\documentclass{article}

\usepackage[backend=bibtex,style=ieee, citestyle=numeric-comp]{biblatex}
\usepackage{arxiv}

\usepackage[utf8]{inputenc} 
\usepackage[T1]{fontenc}    
\usepackage{hyperref}     
\usepackage{url}           
\usepackage{booktabs}      
\usepackage{amsfonts}       
\usepackage{nicefrac}    
\usepackage{microtype}      
\usepackage{graphicx}
\usepackage{algorithm}

\usepackage[noend]{algorithmic}
\usepackage{epsfig}
\usepackage{amsmath}
\usepackage{amsthm}
\usepackage{amssymb}
\usepackage[font=footnotesize]{subfig}

\usepackage{enumitem}
\usepackage{paralist}
\usepackage{authblk}
\usepackage{xpatch}

\xpatchcmd{\author}{\relax#1\relax}{\relax\detokenize{#1}\relax}{}{}

\newtheorem{theorem}{Theorem}
\newtheorem{lemma}{Lemma}

\title{Joint Semantic Domain Alignment and\\ Target
Classifier Learning for\\ Unsupervised Domain Adaptation}

\author[\empty]{Dong-Dong Chen\textsuperscript{\mdseries{1,2,}}\thanks{Work done while author was an intern at JD AI Research.} }
\author[2]{Yisen Wang}
\author[2]{Jinfeng Yi}
\author[3]{Zaiyi Chen}
\author[1]{Zhi-Hua Zhou}
\affil[1]{National Key Laboratory for Novel Software Technology, Nanjing University}
\affil[2]{JD AI Research}
\affil[3]{School of Computer Science, University of Science and Technology of China}
\affil[ ]{\small{\{chendd, zhouzh\}@lamda.nju.edu.cn, eewangyisen@gmail.com\\
yijinfeng@jd.com, chenzaiyi@mail.ustc.edu.cn}}

\begin{document}

\maketitle

\begin{abstract}
  Unsupervised domain adaptation aims to transfer the classifier learned from 
  the source domain to the target domain in an unsupervised manner.
   With the help of target pseudo-labels, aligning class-level distributions 
   and learning the classifier in the target domain are two widely used objectives. 
   Existing methods often separately optimize these two individual objectives, 
   which makes them suffer from the neglect of the other. However, optimizing 
   these two aspects together is not trivial. To alleviate the above issues, 
   we propose a novel method that jointly optimizes semantic domain alignment 
   and target classifier learning in a holistic way.  
  The joint optimization mechanism can not only 
  eliminate their weaknesses but 
  also complement their strengths.
  The theoretical analysis also verifies the favor of the joint optimization mechanism. 
  Extensive experiments on benchmark datasets show that the proposed method 
  yields the best performance in comparison with the state-of-the-art unsupervised 
  domain adaptation methods.
\end{abstract}

\section{Introduction}
Deep Neural Networks (DNNs) have achieved a great success on many tasks
 such as image classification when a large set of labeled examples 
 are available~\cite{Krizhevsky2012,Simonyan2014,
    Szegedy2015,He2016}. However, in many real-world applications, 
    there are plentiful unlabeled data but very limited labeled data;
     and the acquisition of labels is costly, or even infeasible. 
     Unsupervised domain adaptation is a popular way to address this issue.
      It aims at transferring a well-performing model learned from a source 
      domain to a different but related target domain when the labeled data 
      from the target domain is not available~\cite{Csurka2017}.

    Most efforts on unsupervised domain adaptation devote to reducing the 
    domain discrepancy, such that a well-trained classifier in the source
     domain can be applied to the target domain~\cite{Tzeng2014,Long2015,
     Sun2016,Ganin2015,Tzeng2017}. 
    However, these methods only align the distributions in the domain-level, 
    and fail to consider the class-level relations among the source and target samples. 
    For example, a car in the target domain may be mistakenly aligned to a 
    bike in the source domain. 
    To alleviate the class-level misalignment, 
    semantic domain alignment methods~\cite{Xie2018,Pan2019,Chen2018a,Zhang2018a,Deng2018} 
    that enforce the samples from 
    the same class to be close across domains are proposed.
    However, these domain alignment methods neglect the 
    structures in target domain itself.
    Target classifier learning methods~\cite{Saito2017,Zhang2018} 
    learn target discriminative features by distinguishing the samples 
    in the target domain directly. 
    Nonetheless, they may miss some important supervised information in the source domain.
    Intuitively, a straightforward method is to optimize semantic domain alignment and
     target classifier learning jointly. 
    The joint optimization mechanism can not only eliminate their weaknesses,
     but also complement their strengths. 
    The semantic domain alignment methods enforce the intra-class compactness,
     distinguishing different samples from the target domain. 
    The target classifier learning methods enforce the inter-class discrepancy 
    in the target domain, which in turn help to align the same class samples between 
    two domains.
    However, as shown in Figure~\ref{Fig_1_SDA_TCL} (a) and (b),
    semantic domain alignment works in the \textit{feature} space while target 
    classifier learning works in the \textit{label} space. Thus, optimizing 
    them together is not a trivial task. 
  
    \begin{figure}[t]
      \centering
      \includegraphics[width=0.9\textwidth]{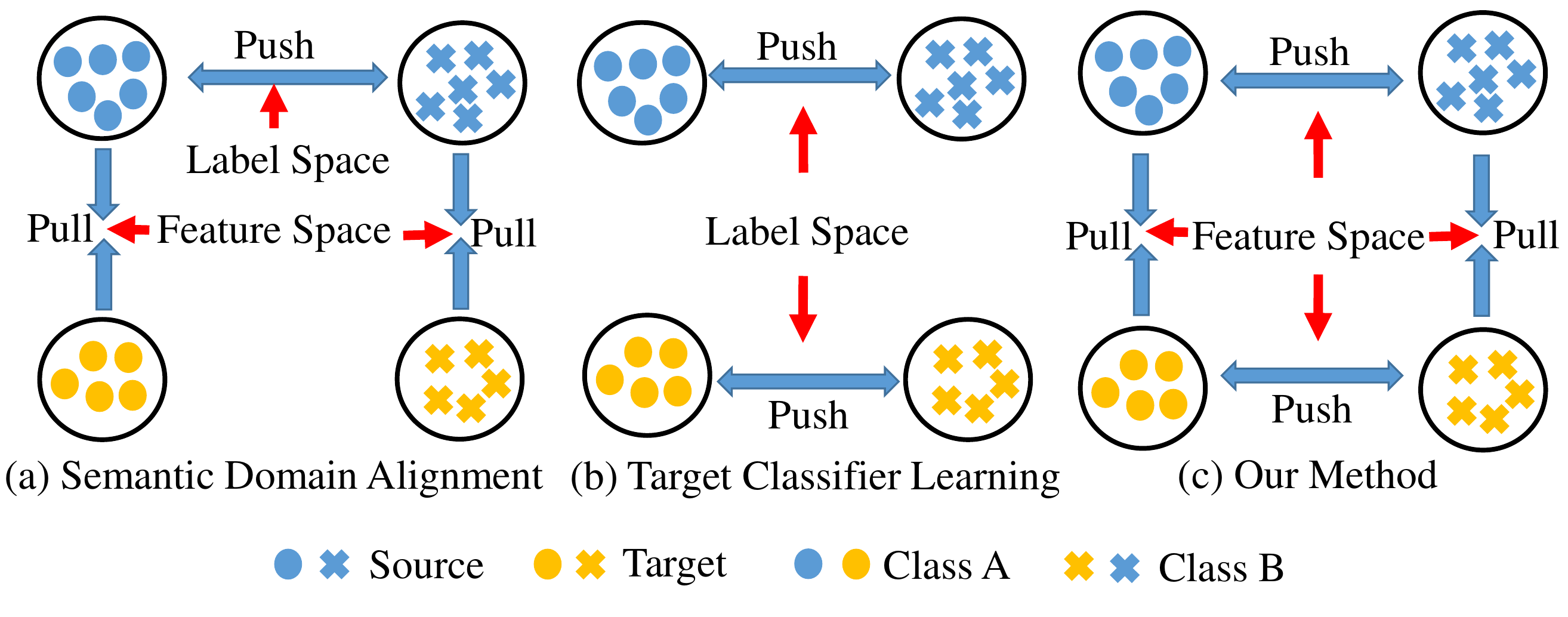}
      \vspace{-0.1 in}
      \caption{Comparisons of semantic domain alignment methods, target 
      classifier learning methods and 
      our proposed method. Note that we jointly optimize 
      semantic domain alignment and target classifier learning in the \textit{feature} 
      space.
      }
      \label{Fig_1_SDA_TCL}
      \vspace{-0.1 in}
  \end{figure} 
    
    In this paper, we propose a novel unsupervised domain adaption method that 
    jointly optimizes semantic domain alignment and target classifier learning 
    in a holistic way. 
    The proposed method is called SDA-TCL, which is short for \textbf{S}emantic 
    \textbf{D}omain \textbf{A}lignment and \textbf{T}arget \textbf{C}lassifier 
    \textbf{L}earning. Figure~\ref{Fig_1_SDA_TCL} (c) illustrats its basic idea. We utilize class centers in the feature space 
    as the bridge to jointly optimize semantic domain-invariant 
    features and target discriminative features both in the feature space.
    For target classifier learning,
    we design the discriminative center loss to learn 
    discriminative features directly by pulling the samples toward
    their corresponding centers according to their pseudo-labels and 
    pushing them away from the other centers.
    For semantic domain alignment, we share the class centers between 
    the same classes across domains to pull the samples from the same class together.
    The main contributions of this paper are as follows:
    \vspace{-0.1cm}
    \begin{compactitem}
        \item To the best of our knowledge, this is the first work trying
         to understand the relationship between semantic domain alignment
          and target classifier learning.
        \item We propose a novel method called \textbf{S}emantic 
    \textbf{D}omain \textbf{A}lignment and \textbf{T}arget \textbf{C}lassifier 
    \textbf{L}earning (SDA-TCL), which can jointly optimize semantic domain 
    alignment and target classifier learning in a holistic way.
        \item We show both theoretically and empirically that the 
        proposed joint optimization mechanism is highly effective.
    \end{compactitem}

\section{Related Work}\label{related_work}

    In this paper, we focus on the problem of deep unsupervised domain 
    adaptation for image classification, and many works along this line of 
    research have been proposed ~\cite{Long2017,Ganin2016,Tzeng2017,Saito2017,Pinheiro2018}. 

    These works can be roughly divided into the following two categories: 
    The first one is to align distributions between the source and the target domain. 
    Its main idea is to reduce the discrepancy between two domains such that a 
    classifier learned from the source domain may be directly applied to the target domain.
    Under this motivation, 
    multiple methods have been used to align the distributions of two domains, 
    such as maximum mean discrepancy (MMD) \cite{Tzeng2014,Long2015,Long2017}, 
    CORrelation ALignment (CORAL) \cite{Sun2016,Chen2018}, 
    attention~\cite{Kang2018}, and optimal transport \cite{Damodaran2018}.
     Besides, adversarial learning is also used to learn domain-invariant 
     features \cite{Ganin2015, Tzeng2017,Long2018,Pinheiro2018}. 
    On par with these methods aligning distributions 
    in the feature space, some methods align distributions
    in raw pixel space by translating source data to the 
    target domain with Image to Image translation 
    techniques~\cite{Liu2016,Bousmalis2017,Zhu2017,
    Liu2017,Hoffman2018,Sankaranarayanan2018}.
    In addition to domain-level distribution alignment, the class-level information 
    in target data is also frequently used to align class-level 
    distributions~\cite{Xie2018,Pan2019,Zhang2018a,Chen2018a,Deng2018,Haeusser2017}. 
    Compared with these methods, 
    our method not only aligns class-level distributions, but also learns target 
    discriminative features. 
    
    The second one is to capture target-specific structures by constructing a reconstruction 
    network \cite{Ghifary2016, Bousmalis2016},
    adjusting the distances between 
    target samples and decision boundaries \cite{Saito2017a,Saito2018,Kumar2018a},
    seeking for density-based
    separations or clusters~\cite{Long2016,French2018,Shu2018a,Sener2016,Laradji2018,
    Liang2018} and learning target classifiers directly~\cite{Saito2017,Zhang2018}.
    Compared with these methods, our method not only learns target classifiers but also 
    aligns class-level distributions, thus is more desirable.

\section{Methodology}\label{our_method}
    In unsupervised domain adaptation, we have a labeled 
    source data set $\mathcal{D}^s = \{(x_i^s,y_i^s)|i=1, 2, \ldots, N^s\}$ 
    and a unlabeled target data set
    $\mathcal{D}^t = \{(x_i^t)|i=1, 2, \ldots, N^t\}$. 
    Suppose the source data have $C$ classes, which is shared with the target data. 
    Our goal is to learn a model from the data set $\mathcal{D}^s \cup \mathcal{D}^t$ 
    to classify the 
    samples in $\mathcal{D}^t$.
    Assume that each class in source (target) data has its corresponding 
    source (target) class 
    center $c_j^s$ ($c_j^t$) ($j \in \mathcal{C} = 
    \{1, 2, \ldots, C\}$) to represent it in the feature space. In our method,
    the target sample $x_i^t$ is classified according to its closest target center in the 
    feature space. 
    A generator network $G$ (parametrized by $\theta_G$) is utilized to generate the 
    features, denoted by $G(x_i^s)$ for the source sample
    $x_i^s$ and $G(x_i^t)$ for the target sample $x_i^t$. 
    
    We aim to jointly optimize semantic domain-invariant features and target 
    discriminative features in the feature space.  
     As illustrated in Figure~\ref{Fig_2_framework_SDA_TCL}, our loss function 
     consists of three parts: 1) $L_{s}(\theta_G)$: 
     It learns discriminative features for source domain by pulling the source sample 
     toward its corresponding source center according to its label
      and pushing it away from the 
     other source centers.
     2) $L_{t}(\theta_G)$: It learns discriminative features for target domain by 
     pulling the target sample toward its corresponding target center 
     according to its pseudo-label
     and pushing it away from the other target centers. 
    3) $L_{c}(\theta_G)$ : 
    It aligns class-level distributions by pulling the source center and the target center
    from the same class.
    We jointly optimize them:
    \begin{equation}\label{loss_for_G}
        L_{G}(\theta_G) =  L_{s}(\theta_G)  
        + \lambda_tL_{t}(\theta_G)+ \lambda_cL_{c}(\theta_G) + \lambda_dL_{d}(\theta_G),
    \end{equation}
    where $\lambda_d$, $\lambda_t$ and $\lambda_c$ are
    the balance parameters and $L_{d}(\theta_G)$ is used to align 
    domain-level distributions for providing a initial classifier to 
    label the pseudo-labels following the previous methods~\cite{Xie2018, Zhang2018a}.
   
     \begin{figure}[t]
        \centering
        \includegraphics[width=0.8\textwidth]{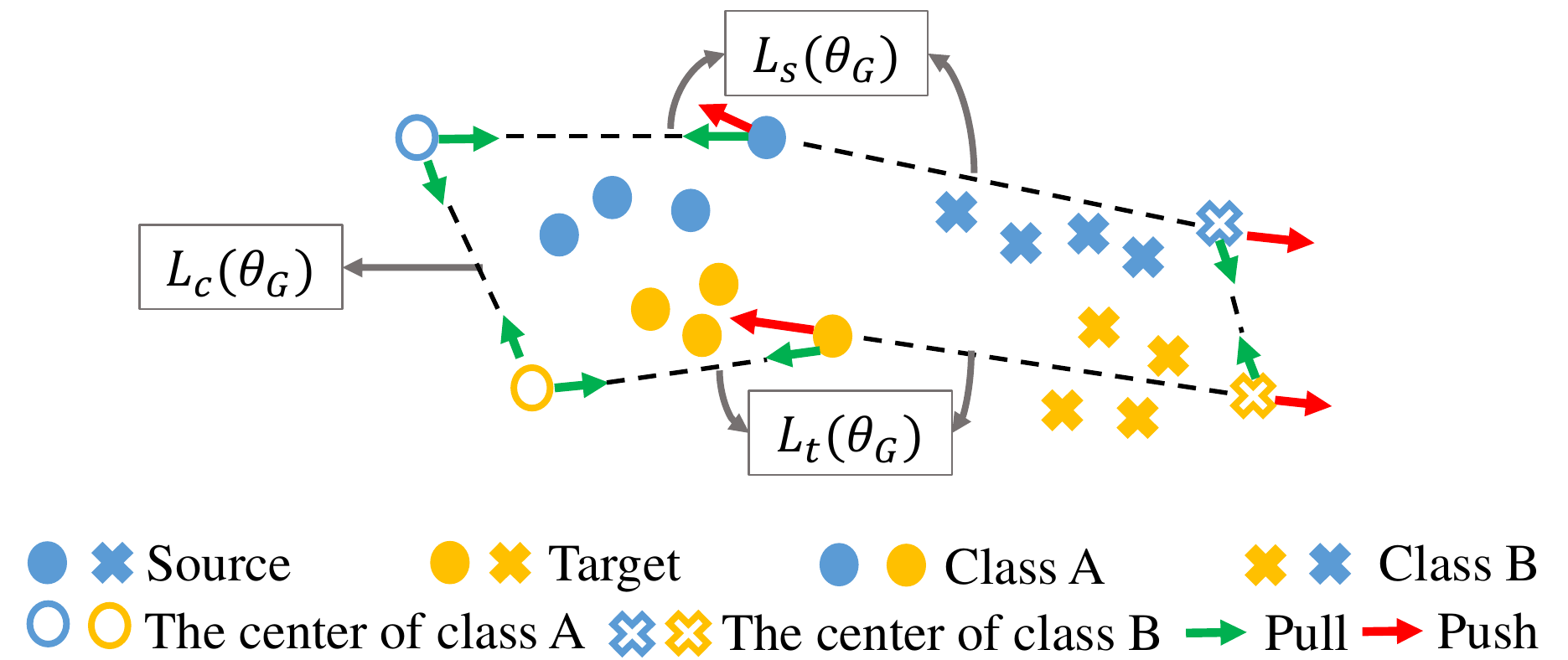}
        \vspace{-0.05 in}
        \caption{Illustration our proposed method SDA-TCL. 
        We jointly optimize semantic domain alignment and 
        target classifier learning in the feature space 
        by optimizing $L_{s}(\theta_G)$, $L_{t}(\theta_G)$ and $L_{c}(\theta_G)$.
        }
        \label{Fig_2_framework_SDA_TCL}
        \vspace{-0.15 in}
    \end{figure}

    \subsection{Learning Source Discriminative Features} \label{source_section}
    We aim to pull the source sample 
    toward its corresponding source center and push it away from the other source centers.
    Here, we design discriminative center loss, which requires that the distances 
    between samples and centers from the same class are smaller than a margin $\alpha$
    and the distances 
    between samples and centers from different classes are larger 
    than a margin $\beta$.
    The discriminative center loss be formulated as:
    \begin{align}\label{center_classification_loss}
        L_{s}(\theta_G) = 
        \sum_{i=1}^{N^s}\Big(
            [d(G(x_i^s),c_{y_i^s}^s) -\alpha]_+
        + [\beta - d(G(x_i^s), c_{\widetilde{y}_i^s}^s)]_+
        \Big),
    \end{align}
    where $d(G(x_i^s), c_j^s)$ 
    denotes the squared Euclidean distance between sample $x_i^s$ and center $c_j^s$, and
    \begin{equation}
        \widetilde{y}_i^s = \mathop{\arg\min} \limits_{j \in \mathcal{C},
         j\neq y_i^s }d(G(x_i^s), c_j^s) 
    \end{equation}
    denotes the closest negative center for sample $x_i^s$ in source centers, and 
    $[a]_+$ denotes the rectifier function which is equal to $\max(0,a)$.
    
    Note that we do not utilize softmax loss
    for classification but 
    design the discriminative center loss.
    The discriminative center loss has two advantages compared with softmax loss:
    1) The discriminative center loss enforces the intra-class compactness, which is helpful to pull ambiguous features away from the class boundaries~\cite{Saito2017a,Saito2018};
    and 2) The discriminative center loss 
    distinguishes the samples in the feature space directly,
     which makes it work in the same space with the class-level alignment.
    
    \subsection{Learning Target Discriminative Features}\label{tcl_sec}
    For the target domain, we aim to learn discriminative features directly in the 
    feature space like source domain.
    Here, we optimize $L_t(\theta_G)$ by utilizing the designed discriminative center loss 
    to pull the 
    target sample toward
    its corresponding target center according to its pseudo-label and
    push it away from the other target centers, which can be formulated as:
    \begin{align}\label{weighted_center_classification_loss}
    \small
        L_t(\theta_G) =\sum_{i = 1}^{N^t}w_i\Big(
            [{d(G(x_i^t), c_{\hat{y}_i^t}^t)} - \alpha]_+ +
        [\beta - d(G(x_i^t), c_{\widetilde{y}_i^t}^t)]_+  
        \Big),
    \end{align}
    where $\hat{y}_i^t$ denotes the pseudo-label for sample $x_i^t$,
     $\widetilde{y}_i^t$ denotes the closest negative target 
    center for sample $x_i^t$ and $w_i$ is the sample weight
    \begin{equation}
        w_i =  \frac{d(G(x_i^t),c_{\widetilde{y}_i^t})}{d(G(x_i^t),c_{\hat{y}_i^t})}-1.
    \end{equation} 
    Then we scale $w_i$ to [0, 1] within the same class.
    A target sample closer to its center than other centers will get 
    a big $w_i$, which
    means the center is more confident on this sample.
    Target pseudo-labels are widely used in the unsupervised domain adaptation 
    methods~\cite{Sener2016,Saito2017,Zhang2018,Xie2018}, while the time to 
    involve pseudo-labels has never been analyzed by these previous methods.
    Involving pseudo-labels from scratch may bring some mistakes by the random 
    pseudo-labels and involving pseudo-labels 
    by a well-learned classifier in the source domain 
    may bring some confident mistakes, which are hard to be corrected.
    We utilize pseudo-labels after a relative small iteration parameter $I_s$ 
    and increase the importance of pseudo-labels by a 
    ramp-up curve (details in Section~\ref{Implementation_Detail_sec}).

    \subsection{Learning Semantic Domain-Invariant 
    Features}\label{semantic_domain_alignment_sec}
    To align class-level distributions, the distances in the feature space between 
    the target samples and the source samples from the same class should be small.
    Constraining the distances between samples directly may bring some noise 
    because of the inaccurate pseudo-labels~\cite{Xie2018}, 
    we alter to
    optimize the distances between the source center and target center from the same class.
    A straightforward method for optimizing $L_{c}(\theta_G)$ can be formulated as:
    \begin{align}\label{class_level_loss}
        L_{c}(\theta_G) = \sum_{j=1}^{C}\left\|c_j^s - c_j^t\right\|_2, 
    \end{align}
    Considering the parameter $\lambda_c$ in Eq.~\ref{loss_for_G} needs to be tuned, we 
    here utilize another method, which makes the class centers are shared between 
    the source domain and target domain, to optimize $L_{c}(\theta_G)$.
    This means that we set 
    \begin{equation}
        c_j^s = c_j^t
    \end{equation}
    for $j \in \mathcal{C} = 
    \{1, 2, \ldots, C\}$ and we do not need to 
    calculate $L_{c}(\theta_G)$ in Eq.~\ref{loss_for_G}. We utilize 
    $\mathcal{C}_s = \{c_j^s\}$ to denote the shared class center set.

    To align domain-level distributions, we adopt the 
    Reverse Gradient (RevGrad) algorithm~\cite{Ganin2015} to construct 
    a discriminator network $D$.
    The discriminator $D$ classifies whether the feature 
    comes from the source or the target domain, and 
    the generator $G$ devotes to fooling $D$, enforcing the generator $G$ to generate 
    domain-invariant features.
    The discriminator $D$ is optimized by the standard classification loss:
    \begin{align}\label{disc_loss}
        L_{d}(\theta_D) = -\sum_{i=0}^{N^s}\log(D(G(x_i^s))) 
        - \sum_{i=0}^{N^t}\log(1- D(G(x_i^t))),
    \end{align}
    while the generator $G$ is optimized to minimize 
    the domain-invariant loss:
    \begin{equation}
        L_{d}(\theta_G) = -L_{d}(\theta_D).    
    \end{equation}

    \subsection{The Complete SDA-TCL Algorithm}\label{algorithm_sec}
    We present the complete procedure of SDA-TCL in Algorithm~\ref{csuda_algorithm}.
    We optimize the generator $G$ and class centers $\{c_j^s\}$ by Eq.~\ref{loss_for_G} and 
    the discriminator $D$ by Eq.~\ref{disc_loss} on each mini-batch.
    As we can see, our objective loss can be computed in linear time.
    We update the pseudo-labels and weights for every $k$ iterations for 
    computational efficiency and we fix $k = 15$ for all experiments.

    \begin{algorithm}[t]
        \begin{small}
            \caption{SDA-TCL}
            \label{csuda_algorithm}
            \textbf{Input:}
            Labeled source set $\mathcal{D}^s$,  
            unlabeled target set  $\mathcal{D}^t$, 
             total iteration $M$, and the frequency to update target pseudo-labels $k$ \\
             \textbf{Output:} The prediction of target data ${\hat{y}_i^t}$
            \begin{algorithmic} [1]
                \floatname{algorithm}{Procedure}
                \renewcommand{\algorithmicrequire}{\textbf{Input:}}
                \renewcommand{\algorithmicensure}{\textbf{Output:}}
                \renewcommand{\algorithmiccomment}{\textbf{asfks}}
                \STATE  \textbf{Initialization:}
                \STATE Randomly initializing the shared center set $\mathcal{C}_s$, generator $G$ and
                 discriminator $D$.
                 \STATE Randomly initializing target label set $\{\hat{y}_i^t\}$, 
                target sample weight set $\{w_i\}$.
                \STATE \textbf{Training:}
                    \FOR {$m=1 \to M$}
                            \STATE Generate training mini-batch $B_m^s$ and $B_m^t$. \\
                            \IF {$(t \mod k) ==0 $}
                                \STATE  Update $\hat{y}_i^t$ and $w_i$ 
                                for $x_i^t \in \mathcal{D}^t$
                                by $\mathcal{C}_s$ and $G$ \label{totel_weight_al}
                            \ENDIF
                            \STATE Train discriminator $D$ 
                            with mini-batch $B_m^s$ and $B_m^t$
                             by minimizing Eq.~\ref{disc_loss}\\
                            \STATE Train generator $G$ and 
                            the shared center set $\mathcal{C}_s$ with 
                            mini-batch $B_m^s$ and $B_m^t$ by minimizing Eq.~\ref{loss_for_G}.\\
                    \ENDFOR
                \STATE \textbf{Inference:}
                \STATE Predicting ${\hat{y}_i^t}$ by 
                generator $G$ and center set $\mathcal{C}_s$
                \renewcommand{\algorithmicrequire}{\textbf{Output:}}
            \end{algorithmic}
        \end{small}
    \end{algorithm}
    
\section{Theoretical Analysis} \label{subsec_analysis}
Following \cite{Ben-David2010}, we theoretically analyze SDA-TCL. 
The following Lemma shows that the upper bound of the expected error 
    on the target samples $\epsilon_\mathcal{T}(h)$ is decided by three terms:
    \begin{lemma} \label{lemma_1}
    Let $\mathcal{H}$ be the hypothesis space.
    Given the source domain $\mathcal{S}$ and target domain $\mathcal{T}$, we have
    \begin{equation} 
    \forall h \in \mathcal{H}, \epsilon_\mathcal{T}(h) \leq \epsilon_\mathcal{S}(h) 
    + \frac{1}{2}d_{\mathcal{H}\Delta\mathcal{H}}(\mathcal{S}, \mathcal{T}) + C,
    \end{equation}
    where the first term $\epsilon_\mathcal{S}(h)$ denotes the expected error on 
    the source samples,
    the second term 
    $\frac{1}{2}d_{\mathcal{H}\Delta\mathcal{H}}(\mathcal{S}, \mathcal{T})$ 
    is the $\mathcal{H}\Delta\mathcal{H}$-distance
    which denotes the divergence between source and target domain, 
    and the third term $C$ is the 
    excepted error of the ideal joint hypothesis. 
    \end{lemma}

    In our method, the first term can be minimized easily with the source labels. 
    Furthermore, the second term 
    is also expected to be small by optimizing the domain-invariant features between 
    $\mathcal{S}$ and $\mathcal{T}$. 
    The third term is treated as a negligibly small 
    term and is
    usually disregarded by previous methods~\cite{Ganin2015, Long2015, Pinheiro2018}.
    However, a large $C$ may hurt the performance on the target domain~\cite{Ben-David2010}.
    We will show that our method optimizes the upper bound for $C$.
    \begin{theorem} \label{theorem_1}
        Let $f_\mathcal{S}$ and $f_\mathcal{T}$ are the labeling functions for domain 
        $\mathcal{S}$ and domain $\mathcal{T}$ respectively. $f_{\hat{T}}$ denotes
        the pseudo target labeling function in our method, we have
        \begin{equation}\label{optimize_C}
            C \leq \min_{h \in \mathcal{H}}
            \epsilon_{\mathcal{S}}(h, f_\mathcal{\hat{T}})
            + \epsilon_{\mathcal{T}}(h, f_\mathcal{\hat{T}}) 
            + \epsilon_{\mathcal{S}}(f_\mathcal{S}, f_\mathcal{\hat{T}})
            + \epsilon_{\mathcal{T}}(f_\mathcal{\hat{T}}, f_\mathcal{T}). 
        \end{equation}
    \end{theorem}
    
    \begin{proof}
    The excepted error of the ideal joint hypothesis $C$ is defined 
    as:
    \begin{equation}
        C = \mathop{\min}\limits_{h \in \mathcal{H}}\epsilon_{\mathcal{S}}(h, f_\mathcal{S}) + \epsilon_{\mathcal{T}}(h, f_\mathcal{T}). 
    \end{equation}
    Following the triangle inequality for classification
     error \cite{Ben-David2006, Crammer2008}, that is, for
      any labeling functions $f_1$, $f_2$ and $f_3$, we
       have $\epsilon(f_1, f_2) \leq \epsilon(f_1, f_3) + \epsilon(f_2, f_3)$,
        we could have 
    \begin{align}
    C &= \min_{h \in \mathcal{H}}\epsilon_{\mathcal{S}}(h, f_\mathcal{S}) 
    + \epsilon_{\mathcal{T}}(h, f_\mathcal{T}) \\ \notag
    & \leq \min_{h \in \mathcal{H}}\epsilon_{\mathcal{S}}(h, f_\mathcal{S})
    + \epsilon_{\mathcal{T}}(h, f_\mathcal{\hat{T}}) 
    + \epsilon_{\mathcal{T}}(f_\mathcal{\hat{T}}, f_\mathcal{T})  \\  \notag
    & \leq \min_{h \in \mathcal{H}}\epsilon_{\mathcal{S}}(h, f_\mathcal{\hat{T}})
    +\epsilon_{\mathcal{T}}(h, f_\mathcal{\hat{T}}) 
    + \epsilon_{\mathcal{S}}(f_\mathcal{S}, f_\mathcal{\hat{T}})
    + \epsilon_{\mathcal{T}}(f_\mathcal{\hat{T}}, f_\mathcal{T}).
    \end{align}
    $\epsilon_{\mathcal{S}}(h, f_\mathcal{\hat{T}}) 
    + \epsilon_{\mathcal{T}}(h, f_\mathcal{\hat{T}})$ denotes the disagreement 
    between $h$ and the pseudo target labeling function $f_\mathcal{\hat{T}}$ and
    is minimized by target classifier 
    learning loss $L_{t}(\theta_G)$ in Eq.~\ref{loss_for_G}.
    $\epsilon_{\mathcal{S}}(f_\mathcal{S}, f_\mathcal{\hat{T}})$ 
    denotes the disagreement between the source
    labeling function $f_\mathcal{S}$ and the pseudo 
    target labeling function $f_\mathcal{\hat{T}}$
    on source samples and is minimized by semantic domain alignment loss $L_{c}(\theta_G)$ 
    in Eq.~\ref{loss_for_G}. 
    Specifically,
    we align class-level distributions by sharing class centers
    between two domains, so a source sample with class $k$ should be predicted 
    as class $k$ by the pseudo target labeling function $f_\mathcal{\hat{T}}$.
    Consequently,  $\epsilon_{\mathcal{S}}(f_\mathcal{S}, f_\mathcal{\hat{T}})$ is 
    expected to be small.
    $\epsilon_{\mathcal{T}}(f_\mathcal{\hat{T}}, f_\mathcal{T})$ 
    denotes the 
    false pseudo-labels ratio, which is assumed to be decreased 
    as the training moves on~\cite{Saito2017,Xie2018}.
    Thus, our method SDA-TCL aims to minimize all the four terms 
    in Theorem~\ref{lemma_1}, while the existing methods neglected 
    the target classifier learning term or the semantic domain alignment
    term~\cite{Saito2017,Zhang2018,Xie2018,Zhang2018a}.
    \end{proof}

\section{Experiments}\label{experimental_result}
    We evaluate the performance of our method on three benchmark unsupervised domain 
    adaptation datasets across different domain shifts
    by classification accuracy metric.

    \subsection{Datasets and Baselines} \label{exp_setup}

    \textbf{Office-31 Dataset}~\cite{Saenko2010}.
    Office-31 dataset contains 4110 images of 31 different categories, 
    which are everyday office objects.
    The images belong to three imbalanced distinct domains: 
    (\romannumeral1) Amazon website (A domain, 2817 images),
    (\romannumeral2) Web camera (W domain, 795 images),
    (\romannumeral3) Digital SLR camera (D domain, 498 images).
    We conduct experiments on the above six transfer tasks.

    \textbf{Digits Datasets} \cite{Hull1994, LeCun1998}. 
    The Digits datasets include USPS~\cite{Hull1994} (U domain) and 
    MNIST~\cite{LeCun1998} (M domain).
    For the tasks U$\rightarrow$M and M$\rightarrow$U, 
    we conduct the experiments on two experimental settings: 1) using all the training 
    data in MNIST and USPS during training~\cite{Bousmalis2017,Pinheiro2018}; 
    2) sampling 2,000 training samples from MNIST and 1,800 training samples from USPS 
    for training~\cite{Tzeng2017}.

    \textbf{VisDA Dataset}~\cite{Peng2017}.
    VisDA dataset evaluates adaptation from synthetic-object to real-object images
     (Synthetic$\rightarrow$Real).
    To date, this dataset represents the largest dataset for cross-domain 
    object classification,
    with 12 categories for three domains.
    In the experiments, we regard the training domain as the source domain and 
    the validation domain as the target domain following the 
    setting in~\cite{Pinheiro2018,Long2018}.

    \textbf{Baseline Methods}.
    We compared our proposed SDA-TCL with state-of-the-art methods: 
    (\uppercase\expandafter{\romannumeral1}) Domain-level distribution alignment methods:
    Gradient Reversal (RevGrad)~\cite{Ganin2015}, 
    Similarity Learning (SimNet)~\cite{Pinheiro2018};
    (\uppercase\expandafter{\romannumeral2}) Class-level distribution alignment methods:
    Transferable Prototypical Networks (TPN)~\cite{Pan2019},
    Domain-Invariant Adversarial Learning (DIAL)~\cite{Zhang2018a},
    Similarity Constrained Alignment (SCA)~\cite{Deng2018};
    (\uppercase\expandafter{\romannumeral3}) Aligning distributions on pixel-level methods:
    Coupled Generative Adversarial Network (CoGAN)~\cite{Liu2016},
    Cycle-Consistent Adversarial Domain Adaptation (CyCADA)~\cite{Hoffman2018},
    Generate To Adapt (GTA)~\cite{Sankaranarayanan2018};
    (\uppercase\expandafter{\romannumeral4}) Utilizing pseudo-labels implicitly methods:
    Maximum Classifier Discrepancy (MCD)~\cite{Saito2018};
    (\uppercase\expandafter{\romannumeral5})Learning target classifier methods:
    Incremental Collaborative and Adversarial Network (iCAN)~\cite{Zhang2018};
    (\uppercase\expandafter{\romannumeral5})State-of-the-art Methods:
    Joint Adaptation Network (JAN)~\cite{Long2017},
    Deep adversarial Attention Alignment (DAAA)~\cite{Kang2018}, 
    Self-Ensembling (S-En)~\cite{French2018},
    Conditional Domain Adversarial Networks (CDAN+E)~\cite{Long2018}.
    With the same protocol, 
    we cite the results from the papers respectively following 
    the previous methods~\cite{Saito2017,Xie2018}. 
    For a better comparison, 
    we report our implementation of the
    RevGrad~\cite{Ganin2015} method, which is denoted as RevGrad-ours.
    We also compare our methods with the Source-only setting, where 
    we train the model by only utilizing the source data.

    \subsection{Implementation Detail}\label{Implementation_Detail_sec}
    \textbf{Network Architectures}.
    For Digits datasets, we construct the generator network $G$ by utilizing
    three convolution layers 
    and a fully-connected layer as the embedding layer following~\cite{Saito2018}.
    For the Office-31 and VisDA dataset, we utilize the 
    ResNet-50~\cite{He2016} network pre-trained on ImageNet~\cite{Russakovsky2015}
    with an embedding layer to 
    represent the generator $G$.
    The discriminator $D$ is a fully-connected network 
    with two hidden layers of 1024 units followed by the domain classifier.

    \textbf{Parameters}.
    We use Adam~\cite{Kingma2014} to optimize class centers, the generator $G$ and 
    discriminator $D$.
    The learning rate are set as $1.0 \times 10^{-4}$ for the networks and 
    $1.0 \times 10^{-2}$ for the class centers respectively.
    We divide the learning rate by 10 when optimizing the pre-trained layers.
    We set the batch size to 32 for each domain and the embedding size to 512.
    For the margin parameters, following~\cite{Schroff15,Manmatha17}, we use the
     recommended value by setting $\alpha = 0.2$ and $\beta = 1.2$.
    For the balance parameters, we set 
    $\lambda_d = \frac{2}{1+\exp(-10 \cdot p)} - 1$ following~\cite{Ganin2016} 
    to suppress noisy signal from the discriminator at the early iterations of training,
    where $p$ is training progress changing from 0 to 1.
    We set $\lambda_t = K \times \lambda_d$ to focus more on the target pseudo-labels 
    as the training process goes on.
    We choose the parameter 
    $K =5$ and the time $I_s = 200$ for involving pseudo-labels via 
    reverse cross-validation~\cite{Ganin2016} on the task D $\rightarrow$ A and 
    fix them for the experiments.
    We run all experiments with 
    PyTorch on a Tesla V100 GPU.
    We repeat each experiment 5 times and report mean accuracy and standard deviation.

    \subsection{Results}
    The results of SDA-TCL on the Office-31, Digits and 
    VisDA datasets are shown in Table~\ref{office-31_exp}
    and Table~\ref{digit_visda_exp}.
    Compared with the Source-only setting, 
    SDA-TCL improves the performance by utilizing the unlabeled target data in 
    all transfer tasks. 
    It improves the average absolute accuracy by 23.0\% in digits experiments, 
    12.2\% in Office-31 experiments and 22.3\% in VisDA experiments.
    
    Compared with the semantic domain alignment methods (TPN~\cite{Pan2019}, 
    DIAL~\cite{Zhang2018a}, SCA~\cite{Deng2018}) and 
    the target classifier learning methods 
    (iCAN~\cite{Zhang2018}),
    our method outperforms them in most transfer tasks by jointly optimizing 
    semantic domain alignment and 
    target classifier learning in the feature space.
    Compared with state-of-the-art methods, 
    SDA-TCL achieves better or comparable performance in all transfer tasks.
    Please note that S+En\cite{French2018} averaged predictions of 16 differently 
    augmentations versions
    of each image to achieve the accuracy 82.8\% on VisDA dataset while SDA-TCL achieves
    the accuracy 81.9\% by making only one prediction
    for each image following the most methods.
    It is desirable that SDA-TCL outperforms other methods by a large margin in hard tasks, 
    e.g., W$\rightarrow$A and D$\rightarrow$A. 
    Note that we do not tune parameters for every dataset and 
    our results can be improved 
    further by choosing parameters carefully, which is shown in 
    Section~\ref{exp_analysis_sec}.

    \begin{table}[t]
        \caption{Accuracy (\%) for the Office-31 dataset.}
        \label{office-31_exp}
        \begin{center}
        \begin{small}
        \begin{tabular}{lccccccc}
        \toprule
        Method &A$\rightarrow$W&D$\rightarrow$W & W$\rightarrow$D & A$\rightarrow$D &
        D$\rightarrow$A&W$\rightarrow$A &Average\\
        
        \midrule
        RevGrad~\cite{Ganin2015}                & 82.0$\pm$0.4 & 96.9$\pm$0.2 & 99.1$\pm$0.1 
                                                & 79.7$\pm$0.4 & 68.2$\pm$0.4 & 67.4$\pm$0.5 & 82.2 \\
        JAN~\cite{Long2017}                     & 85.4$\pm$0.3 & 97.4$\pm$0.2 & 99.8$\pm$0.2 
                                                & 84.7$\pm$0.3 & 68.6$\pm$0.3 & 70.0$\pm$0.4 & 84.3 \\
        GTA~\cite{Sankaranarayanan2018}         & 89.5$\pm$0.5 & 97.9$\pm$0.3 & 99.8$\pm$0.4 
                                                & 87.7$\pm$0.5 & 72.8$\pm$0.3 & 71.4$\pm$0.4 & 86.5 \\
        DAAA~\cite{Kang2018}                    & 86.8$\pm$0.2 & \textbf{99.3}$\pm$0.1 & \textbf{100.0}$\pm$0.0 
                                                & 88.8$\pm$0.4 & 74.3$\pm$0.2 & 73.9$\pm$0.2 & 87.2 \\                                       
        DIAL~\cite{Zhang2018a}                  & 91.7$\pm$0.4 & 97.1$\pm$0.3 & 99.8$\pm$0.0 
                                                & 89.3$\pm$0.4 & 71.7$\pm$0.7 & 71.4$\pm$0.2 & 86.8 \\
        iCAN~\cite{Zhang2018}                   & 92.5   & 98.8 & \textbf{100.0} 
                                                & 90.1   & 72.1 & 69.9    & 87.2  \\ 
        SCA~\cite{Deng2018}                     & 93.5   & 97.5 & \textbf{100.0} 
                                                & 89.5   & 72.4 & 72.7    & 87.6  \\
        CDAN+E~\cite{Long2018}                  & \textbf{94.1}$\pm$0.1   & 98.6$\pm$0.1 & \textbf{100.0}$\pm$0.0 
                                                & 92.9$\pm$0.2   & 71.0$\pm$0.3 & 69.3$\pm$0.3    & 87.7  \\
        \midrule
        Source-only       & 72.3$\pm$0.8 & 96.5$\pm$0.7 & 99.1$\pm$0.5 & 80.7$\pm$0.5 
                            & 59.7$\pm$1.2 & 59.7$\pm$1.5 & 78.0 \\ 
        RevGrad-ours    & 83.5$\pm$0.5 & 96.8$\pm$0.2 & 99.2$\pm$0.5 
                            & 83.2$\pm$0.4 & 67.6$\pm$0.4 & 65.8$\pm$0.6 & 82.7 \\
        SDA-TCL & 92.4$\pm$0.7 & 99.1$\pm$0.1  & \textbf{100.0}$\pm$0.0 
                & \textbf{93.2}$\pm$1.2  & \textbf{79.0}$\pm$0.3 
                & \textbf{77.6}$\pm$1.0 & \textbf{90.2}\\
        \bottomrule
        \end{tabular}
        \end{small}
        \end{center}
        
    \end{table}

    \begin{table}[t]
        \caption{Accuracy (\%) for the Digit datasets and VisDA dataset.
        $*$ means the setting that utilizes all the training data.
        $\dag$ indicates that this method uses multiple data augmentations.}
        \label{digit_visda_exp}
        \begin{center}
        \begin{small}
        \begin{tabular}{lcccc|lc}
        \toprule
        Method & U$\rightarrow$M & M$\rightarrow$U & $\text{U}^*\rightarrow\text{M}^*$ 
        & $\text{M}^*\rightarrow\text{U}^*$ &
         Method & Synthetic$\rightarrow$Real  \\
        \midrule
        CoGAN~\cite{Liu2016}              & 89.1$\pm$0.8   & 91.2$\pm$0.8 & 93.2          & 95.7             
        &JAN~\cite{Long2017}                           & 61.6    \\
        CyCADA~\cite{Hoffman2018}         &  -   & - & 96.5$\pm$0.1          &  95.6$\pm$0.2
        &GTA~\cite{Sankaranarayanan2018}               & 69.5    \\
        DIAL~\cite{Zhang2018a}            & 97.3$\pm$0.3   & 95.0$\pm$0.2 &  \textbf{99.1}$\pm$0.1&97.1$\pm$0.2  
        &SimNet~\cite{Pinheiro2018}                    & 69.6    \\
        MCD~\cite{Saito2018}        & 94.1$\pm$0.3   & 94.2$\pm$0.7 & -& 96.5$\pm$0.3
        &MCD~\cite{Saito2018}                    & 71.9    \\
        CDAN+E~\cite{Long2018}               & -              & - & 98.0   & 95.6
        &CDAN+E~\cite{Long2018}                          & 70.0    \\
        TPN~\cite{Pan2019}          & 94.1 & 92.1& - & -
        &TPN~\cite{Pan2019}           & 80.4    \\ 
        \midrule
        Source-only                       & 71.9$\pm$2.3 & 78.1$\pm$3.5 & 70.5$\pm$1.9 & 80.3$\pm$1.7
        &Source-only                                   & 59.6$\pm$0.2    \\ 
        SDA-TCL                      & \textbf{97.6}$\pm$0.2
                                    & \textbf{97.6}$\pm$0.4  
                             & 99.0$\pm$0.1 & \textbf{98.9}$\pm$0.1 
        &SDA-TCL                              & \textbf{81.9}$\pm$0.3 \\
         \midrule
         S-En~\cite{French2018}
         & - & - & 98.1$\pm$2.8 & 98.3$\pm$0.1 
        &S-En~\cite{French2018}& 74.2\\
         $\text{S-En}^\dag$~\cite{French2018}
         & - & - & 99.5$\pm$0.0 & 98.2$\pm$0.1 
        &$\text{S-En}^\dag$~\cite{French2018}& 82.8\\
        
        \bottomrule
        \end{tabular}
        \end{small}
        \end{center}

    \end{table}

    \subsection{Ablation Study}\label{Ablation_study}
    Our method is not a straightforward
    combination of semantic domain alignment methods and 
    target classifier learning methods.
    Existing methods~\cite{Saito2017,Zhang2018} utilize two different losses
     to learn the target 
    discriminative features (softmax loss) and semantic 
    domain-invariant features (center alignment loss \cite{Zhang2018a, Xie2018, Chen2018a}). Instead, 
    We design the discriminative center loss and share the class centers to carry
    out the joint optimization mechanism in the same space.
    Here, We implement the origin SDA and TCL with softmax loss, denoted as 
    SDA-origin and TCL-origin respectively.
    We also implement a linear combination of these two origin methods, denoted as 
    Linear-Combination. 
    For a better comparison, We further conduct experiments on our method 
    without semantic domain alignment (TCL-ours) and 
    without target classifier learning (SDA-ours), respectively.
    The results are shown in Table~\ref{target_loss_diff}.
    
    There are several interesting observations:
    (1) SDA-ours and TCL-ours often show different superiority on different tasks, 
    which means they benefit from the target pseudo-labels
    from different aspects. As a result, 
     the joint optimization SDA-TCL shows better results than only optimizing one of them.
    (2)  When comparing TCL-ours and TCL-origin, 
    TCL-ours outperforms TCL-origin in the transfer tasks, which may benefit from 
    the features optimized by discriminative center loss having intra-class compactness.
    When comparing SDA-ours and SDA-origin, SDA-ours shows better results than 
    SDA-origin, which may be owed that the features in SDA-ours are optimized 
    in the same space. 
    These observations, which are consistent with 
    the analysis in Section~\ref{source_section}, show the effectiveness 
    of discriminative center loss.
    (3) The Linear-Combination does not show any 
    advantages while our method 
    SDA-TCL can highlight it.
    Because Linear-Combination optimizes 
    the features in separate space and it is 
    more sensitive to the weight balance parameters 
    compared with our holistic method SDA-TCL in the 
    experiments.

     \begin{table}[t]
        \caption{Accuracy (\%) for the Office-31 dataset and 
        VisDA dataset under different settings.}
        \label{target_loss_diff}
        \begin{center}
        \begin{small}
        \begin{tabular}{lccccc|c}
        \toprule
        Method &A$\rightarrow$W & A$\rightarrow$D &
        D$\rightarrow$A&W$\rightarrow$A &Average  & Synthetic$\rightarrow$Real\\
        \midrule
        TCL-origin         & 89.6$\pm$0.8 
                     & 86.9$\pm$0.7 & 72.3$\pm$0.4 & 68.7$\pm$0.6 & 79.4&70.8$\pm$0.5 \\
        SDA-origin         & 89.2$\pm$0.7 
                     & 88.3$\pm$1.0 & 72.9$\pm$0.5 & 70.5$\pm$0.6 & 80.2 &68.4$\pm$0.5\\
        Linear-Combination            & 89.4$\pm$0.9 
                     & 87.2$\pm$0.6 & 73.3$\pm$0.5 & 71.5$\pm$0.5 & 80.4  &70.4$\pm$0.6\\
         \midrule
        TCL-ours          & 90.0$\pm$1.7 
                     & 92.2$\pm$2.3 & 77.7$\pm$0.6 & 77.1$\pm$1.8 & 84.3 &81.5$\pm$0.8 \\
        SDA-ours         & \textbf{92.8}$\pm$0.8 
                     & 92.7$\pm$1.2 & 77.6$\pm$0.5 & 76.9$\pm$0.9 & 85.0  &79.8$\pm$0.8\\
       
        SDA-TCL      & 92.4$\pm$0.7 
        & \textbf{93.2}$\pm$1.2 & \textbf{79.0}$\pm$0.3 
        & \textbf{77.6}$\pm$1.0 & \textbf{85.6}  & \textbf{81.9}$\pm$0.3\\

        \bottomrule
        \end{tabular}
        \end{small}
        \end{center}
    \end{table}
    
    \begin{figure}[t]
            \centering
            \subfloat[Parameter $I_s$]{%
            \includegraphics[width=0.24\textwidth,height=2.8cm,trim={0cm 0cm 0cm 0cm},clip]
            {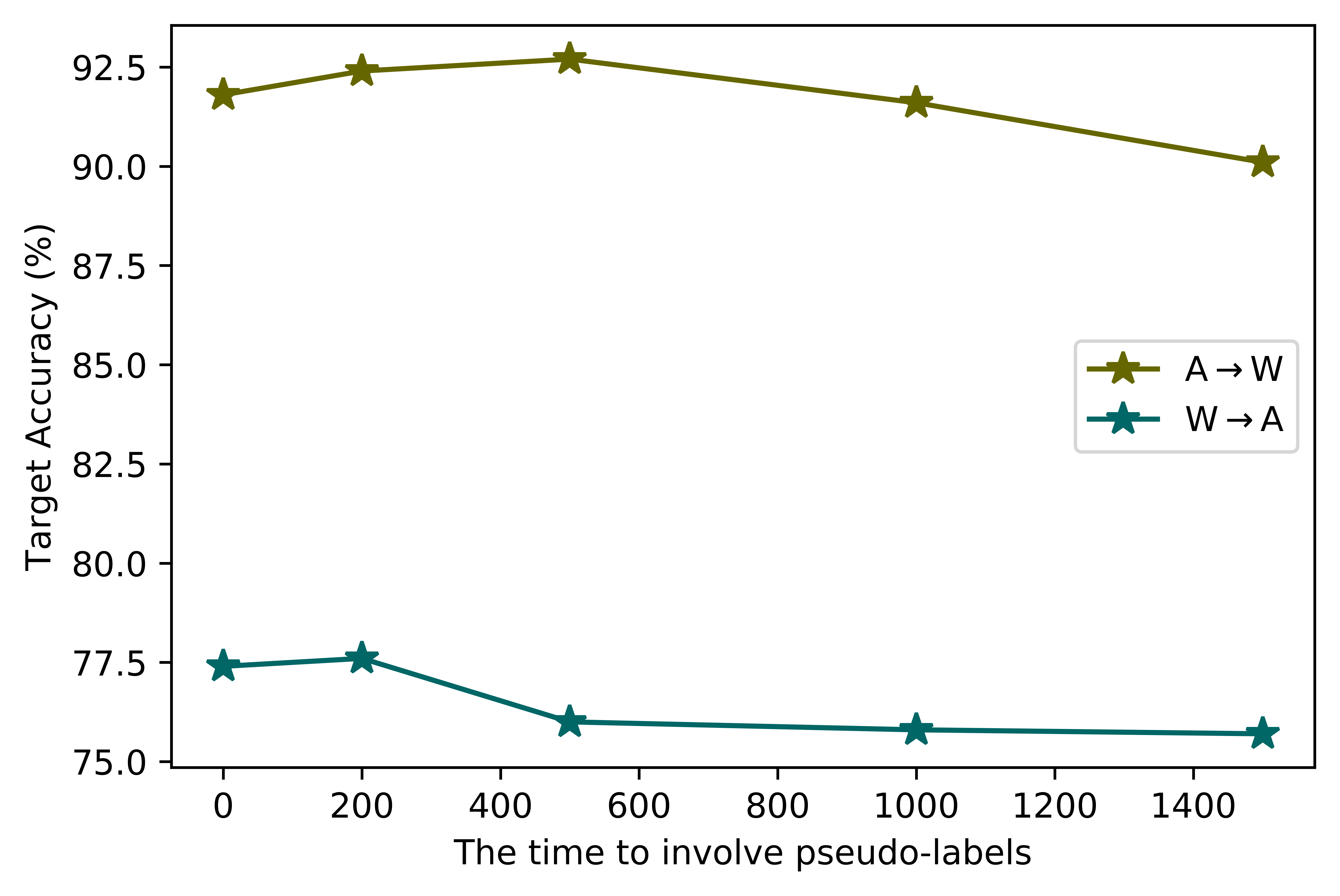}}
            \subfloat[Parameter $K$]{%
            \includegraphics[width=0.24\textwidth,height=2.8cm,trim={0cm 0cm 0cm 0cm},clip]
            {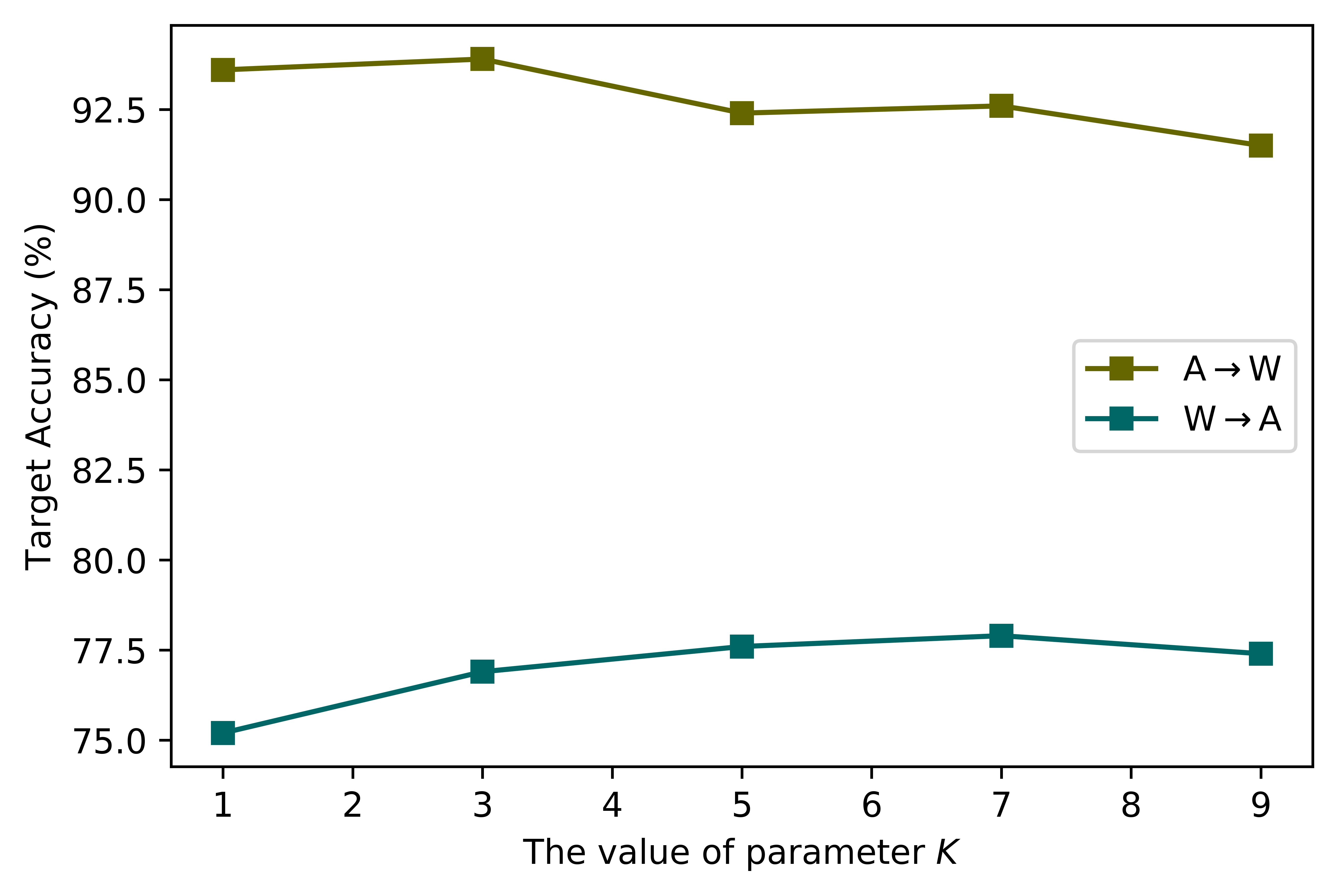}}
            \subfloat[Discrepancy]{%
            \includegraphics[width=0.24\textwidth,height=2.8cm,
            trim={0cm 0cm 0cm 0cm},clip]{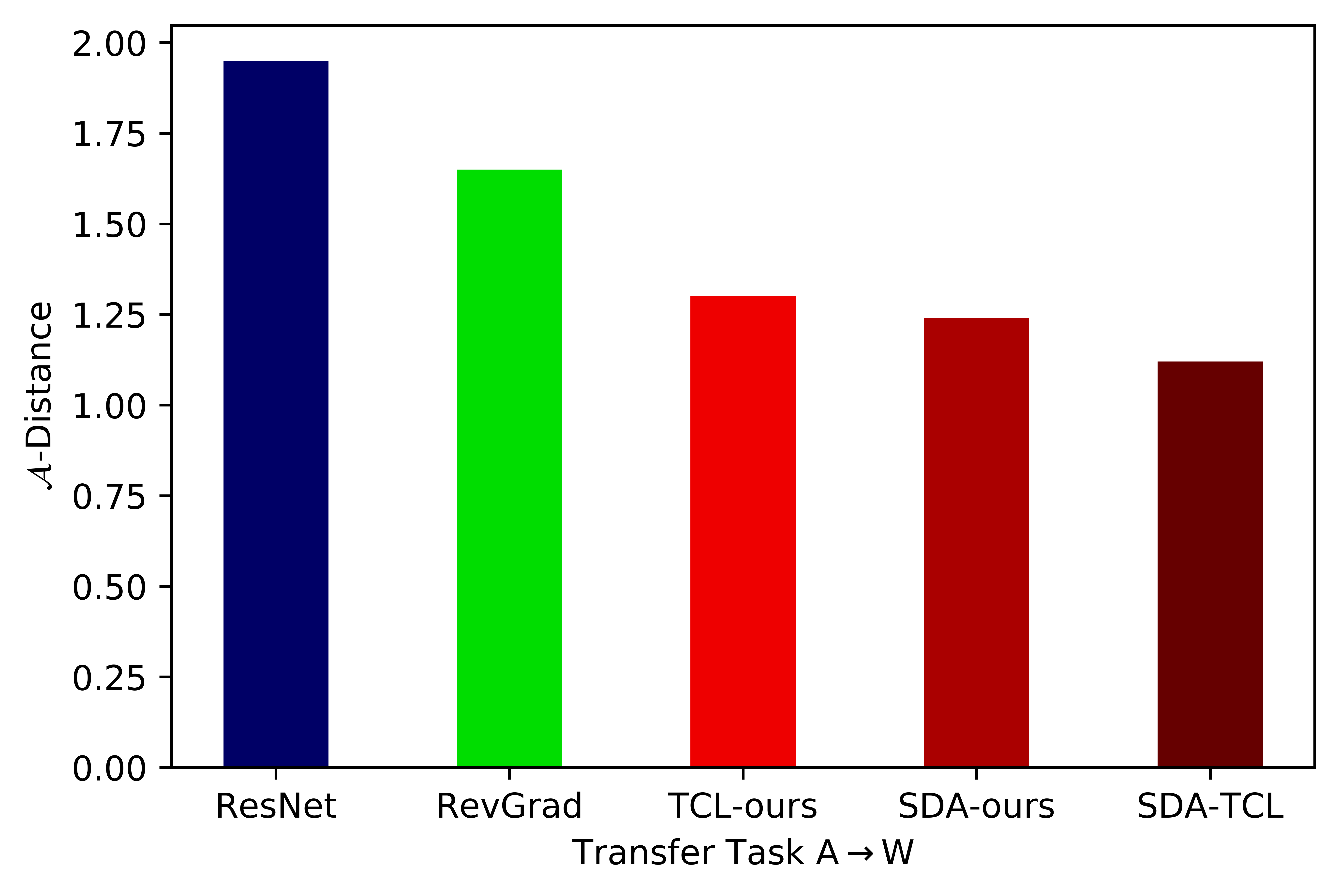}}
            \subfloat[Convergence]{%
            \includegraphics[width=0.24\textwidth,height=2.8cm,
            trim={0cm 0cm 0cm 0cm},clip]{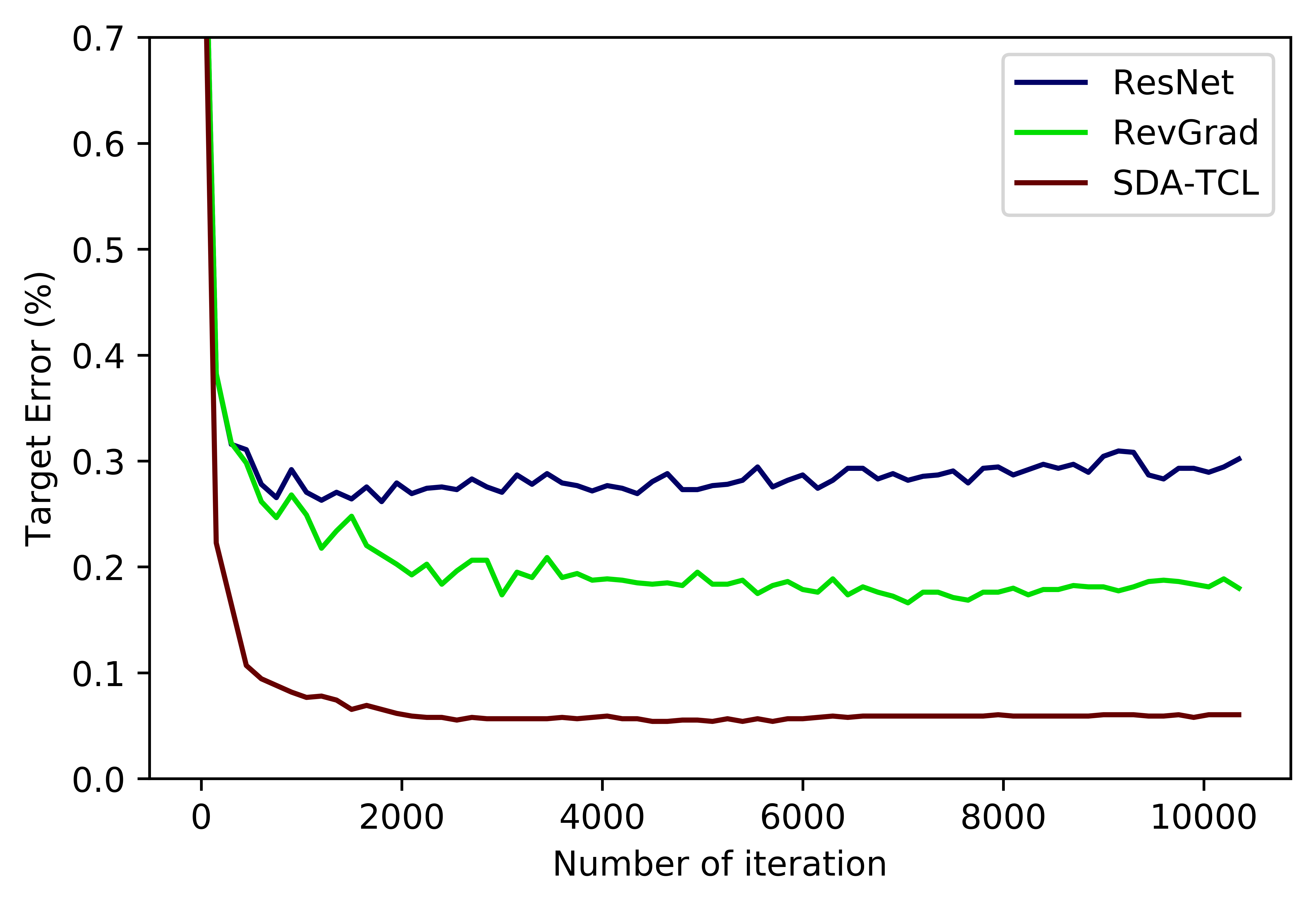}}
            \caption{
                Analysis of 
                parameter $I_s$, 
                parameter $K$, distribution discrepancy, and convergence.
            }
            \label{feat_vis_pic}
    \end{figure}

    \subsection{Empirical Understanding}\label{exp_analysis_sec}

    \textbf{The time to involve pseudo-labels}. 
    We utilize the parameter $I_s$ to control the time to involve 
    the target pseudo-labels in Section~\ref{tcl_sec} and we conduct 
    experiments by choosing $I_s$ from \{0, 200, 500, 1000, 1500\}
    on task A$\rightarrow$W and W$\rightarrow$A.
    The results shown in Figure~\ref{feat_vis_pic}(a)
    indicate that 
    there is a trade-off for the time to involve target pseudo-labels
    and a relative small iteration could be a good choice, which 
    is consistent with the analysis in Section~\ref{tcl_sec}.
    
    \textbf{Parameter Sensitivity}. In our method SDA-TCL, we use the parameter $K$ 
    to decide $\lambda_t$ that controls 
    the importance of  
    utilizing the target pseudo-labels.
    We conduct experiments to evaluate SDA-TCL by
    choosing $K$ in the range of \{1,3,5,7,9\} on task A$\rightarrow$W and W$\rightarrow$A.
    From the results shown in Figure~\ref{feat_vis_pic}(b), we can find 
    that SDA-TCL can achieve good performance with a wide range of $K$.

    \textbf{Distribution Discrepancy}. 
    The $\mathcal{A}$-distance is defined
     as $\text{dist}_{\mathcal{A}} = 2(1-2\epsilon)$
    to measure the distribution discrepancy~\cite{Ben-David2010, Mansour2009}, where 
    $\epsilon$ denotes the test error of a classifier trained to discriminate 
    the source from target. 
    A smaller $\text{dist}_{\mathcal{A}}$ means a smaller domain gap.
    Figure~\ref{feat_vis_pic}(c) shows $\text{dist}_{\mathcal{A}}$ on 
    task A$\rightarrow$W with features of ResNet, RevGrad, SDA-ours, 
    TCL-ours and SDA-TCL. The results indicate that SDA-TCL can reduce the domain 
    gap more effectively. With class-level distribution alignment,
    SDA-ours and SDA-TCL have a smaller $\text{dist}_{\mathcal{A}}$ 
    than RevGrad. TCL-ours also has a smaller $\text{dist}_{\mathcal{A}}$ 
    than RevGrad, which indicates that TCL-ours is helpful for the domain 
    alignment.

    \textbf{Convergence}. 
    We demonstrate the convergence of ResNet, RevGrad, and SDA-TCL, 
    with the error rates in the target domain on task A$\rightarrow$W 
    shown in Figure~\ref{feat_vis_pic}(d).
    SDA-TCL has faster convergence than RevGrad and the convergence process 
    is more stable than RevGrad.

\section{Conclusion}\label{conclusion}
    In this paper, we proposed a novel method for unsupervised domain adaptation
    by jointly optimizing semantic domain alignment and 
    target classifier learning in the \textit{feature} space.
    The joint optimization mechanism can not only 
  eliminate their weaknesses but 
  also complement their strengths.
    Experiments on several benchmarks demonstrate that our method
    surpasses state-of-the-art unsupervised domain adaptation methods.
    Recently, learnware is defined to be facilitated with model reusability \cite{Zhou2016}. 
    The use of a learned model to another task, however, is not trivial. 
    There have been some efforts towards this 
    direction~\cite{Ye2018,Shen2018,Yang2017,Hu2018}, 
    whereas the approach presented in this paper offers another possibility.

\printbibliography

\end{document}